\documentclass{article}
\usepackage{graphicx}
\usepackage[utf8]{inputenc}
\usepackage{kotex}
\usepackage{amsfonts}
\usepackage{amsmath}
\usepackage{amsthm}
\usepackage{amssymb}
\usepackage{tikz}
\usepackage{caption}
\usepackage{float}
\usepackage{mathtools}
\usepackage{mathrsfs}
\usepackage{algorithm2e}

\newcommand{\norm}[1]{\left\lVert#1\right\rVert}

\usepackage[margin=1in]{geometry}

\DeclareMathOperator{\Var}{Var}

\DeclareMathOperator{\supp}{supp}
\DeclareMathOperator{\diag}{diag}

\DeclareMathOperator{\Hess}{Hess}
\DeclareMathOperator{\iso}{iso}
\DeclareMathOperator{\geo}{geo}
\DeclareMathOperator{\diam}{diam}

\DeclareMathOperator*{\argmax}{arg\,max}
\DeclareMathOperator*{\argmin}{arg\,min}

\makeatletter
\DeclareFontFamily{OMX}{MnSymbolE}{}
\DeclareSymbolFont{MnLargeSymbols}{OMX}{MnSymbolE}{m}{n}
\SetSymbolFont{MnLargeSymbols}{bold}{OMX}{MnSymbolE}{b}{n}
\DeclareFontShape{OMX}{MnSymbolE}{m}{n}{
	<-6>  MnSymbolE5
	<6-7>  MnSymbolE6
	<7-8>  MnSymbolE7
	<8-9>  MnSymbolE8
	<9-10> MnSymbolE9
	<10-12> MnSymbolE10
	<12->   MnSymbolE12
}{}
\DeclareFontShape{OMX}{MnSymbolE}{b}{n}{
	<-6>  MnSymbolE-Bold5
	<6-7>  MnSymbolE-Bold6
	<7-8>  MnSymbolE-Bold7
	<8-9>  MnSymbolE-Bold8
	<9-10> MnSymbolE-Bold9
	<10-12> MnSymbolE-Bold10
	<12->   MnSymbolE-Bold12
}{}

\let\llangle\@undefined
\let\rrangle\@undefined
\DeclareMathDelimiter{\llangle}{\mathopen}%
{MnLargeSymbols}{'164}{MnLargeSymbols}{'164}
\DeclareMathDelimiter{\rrangle}{\mathclose}%
{MnLargeSymbols}{'171}{MnLargeSymbols}{'171}
\makeatother

\begin{document}

\newtheorem{thm}{Theorem}
\newtheorem{cor}{Corollary}
\newtheorem{lemma}{Lemma}
\newtheorem{prop}{Proposition}

\theoremstyle{remark}
\newtheorem{rmk}{Remark}

\theoremstyle{definition}
\newtheorem{defn}{Definition}

\captionsetup[figure]{labelfont={bf},name={Figure},labelsep=period}

\title{\Large\textbf{Hessian Based Smoothing Splines for Manifold Learning}}
\author{Juno Kim}
\date{Dept. of Statistics, Seoul National University, South Korea\\[2ex]
	December 23, 2022}

\maketitle

\begin{abstract}
We propose a multidimensional smoothing spline algorithm in the context of manifold learning. We generalize the bending energy penalty of thin-plate splines to a quadratic form on the Sobolev space of a flat manifold, based on the Frobenius norm of the Hessian matrix. This leads to a natural definition of smoothing splines on manifolds, which minimizes square error while optimizing a global curvature penalty. The existence and uniqueness of the solution is shown by applying the theory of reproducing kernel Hilbert spaces. The minimizer is expressed as a combination of Green's functions for the biharmonic operator, and `linear' functions of everywhere vanishing Hessian. Furthermore, we utilize the Hessian estimation procedure from the Hessian Eigenmaps algorithm to approximate the spline loss when the true manifold is unknown. This yields a particularly simple quadratic optimization algorithm for smoothing response values without needing to fit the underlying manifold. Analysis of asymptotic error and robustness are given, as well as discussion of out-of-sample prediction methods and applications.
\end{abstract}

\section{Introduction}

The usual smoothing penalty associated to a $d$-dimensional thin-plate spline $f: \mathbb{R}^d \rightarrow \mathbb{R}$ is given as:
\begin{equation*}
	\mathcal{H}(f)=\int_{\mathbb{R}^d} \sum_{i,j=1}^d \left( \frac{\partial^2 f(\mathbf{x})}{\partial x_i \partial x_j}\right)^2 d\mathbf{x}
\end{equation*}
physically corresponding to the bending energy of the surface defined by $f$. Note that the integrand is the Frobenius norm of the Hessian matrix $\Hess_f(\mathbf{x})$ of $f$ at $\mathbf{x}$, so this penalty easily generalizes to functions on a flat $d$-dimensional Riemannian manifold $\mathcal{M}$, possibly with boundary:
\begin{equation*}
	\mathcal{H}^{\mathcal{M}}(f):= \int_{\mathcal{M}} \norm{\Hess_f(p)}_F^2 d\mu(p)
\end{equation*}
which is a quadratic form defined on the Sobolev space $W_2^2(\mathcal{M})$. The flatness of $\mathcal{M}$ is required in order for the Frobenius norm of the Hessian tensor to be well-defined in local orthonormal coordinates, and to avoid dealing with estimating curvature terms.

Suppose we are given observations $(p_i,y_i)$, $i=1,\cdots,N$, where $p_i\in\mathcal{M}$ and $y_i\in \mathbb{R}$. The data follow the rule $y_i=f(p_i) + \epsilon_i$, where the smooth function $f:\mathcal{M} \rightarrow \mathbb{R}$ is to be estimated, and the error terms are i.i.d. normally distributed with variance $\sigma^2$. We define the \textit{smoothing spline estimator} with smoothing parameter $\lambda$ as follows:
\begin{equation*}
	\hat{f}^{SS} = \argmin_{f\in W_2^2(\mathcal{M})} \mathcal{L}^{SS}(f;\lambda), \quad \mathcal{L}^{SS}(f;\lambda) = \sum_{i=1}^N (y_i-f(p_i))^2 + \lambda\cdot \mathcal{H}^{\mathcal{M}}(f)
\end{equation*}

In Sections 3 and 4, we study the minimizer analytically using the theory of reproducing kernel Hilbert spaces and derive $\hat{f}^{SS}$ in terms of Green's functions on $\mathcal{M}$.

\bigskip
Now consider the situation where $\mathcal{M}$ is unknown, and we are instead given data points $\mathbf{x}_i = \phi(p_i)$ embedded in some high-dimensional Euclidean feature space $\mathbb{R}^n$, $n>d$, via a locally isometric embedding $\phi: \mathcal{M} \rightarrow \mathbb{R}^n$. The observed values are now $y_i = g(\mathbf{x}_i) + \epsilon_i$, $g=f\circ \phi^{-1}$, and since both $\phi$ and $\mathcal{M}$ are unknown we cannot directly apply the above smoothing spline formulation.

However, the Hessian Eigenmaps algorithm in manifold learning \cite{hlle} provides us with a powerful method to estimate $\mathcal{H}^{\mathcal{M}}(f)$ using only the given data $\mathbf{X} = (\mathbf{x}_1 \cdots \mathbf{x}_N)^{\top}$, $\mathbf{y} = (y_1,\cdots,y_N)^{\top}$. We may compute a matrix $\hat{\mathcal{H}}^{\mathcal{M}}(\mathbf{X})$ approximating $\mathcal{H}^{\mathcal{M}}$ in the sense that:
\begin{equation*}
	\mathcal{H}^{\mathcal{M}}(f) \approx g(\mathbf{X})^{\top} \hat{\mathcal{H}}^{\mathcal{M}}(\mathbf{X}) g(\mathbf{X})
\end{equation*}
for any pair of functions $g=f\circ \phi^{-1}$ where $g(\mathbf{X}) = (g(\mathbf{x}_1),\cdots, g(\mathbf{x}_N))^{\top}$. In \cite{hlle}, this approximation is utilized to find the nullspace of $\mathcal{H}^{\mathcal{M}}$, which consists of the coordinate functions in the case where $\mathcal{M}$ is a subspace of $\mathbb{R}^d$. However, we do not need this additional assumption.

Thus, we may replace the loss $\mathcal{L}^{SS}(f;\lambda)$ by the ``Hessian spline'' loss
\begin{equation*}
	\mathcal{L}^{HS}(g;\lambda) := \norm{\mathbf{y} - g(\mathbf{X})}_2^2 + \lambda g(\mathbf{X})^{\top} \hat{\mathcal{H}}^{\mathcal{M}}(\mathbf{X}) g(\mathbf{X})
\end{equation*}
defined for functions $g: \phi(\mathcal{M}) \rightarrow \mathbb{R}$. The loss is quadratic and depends only on the values $g(\mathbf{X})$, so that the fitted values $\hat{g}^{HS}(\mathbf{X})$ may be computed simply as:
\begin{equation*}
	\hat{g}^{HS}(\mathbf{X}) = (I_N + \lambda \hat{\mathcal{H}}^{\mathcal{M}}(\mathbf{X}))^{-1} \mathbf{y}
\end{equation*}
and out-of-sample estimates $\hat{g}(\mathbf{x})$ for $\mathbf{x}$ (assumed to be) lying on $\phi(\mathcal{M})$ may be calculated using various neighborhood interpolation techniques.

If we are additionally given weights $\mathbf{w}=(w_1,\cdots,w_N)^{\top}$ representing the reliability of each datum $y_i$, the weighted loss
\begin{equation*}
	\mathcal{L}^{HS}(g;\lambda,\mathbf{w}) = \sum_{i=1}^N w_i(y_i-g(\mathbf{x}_i))^2 + \lambda g(\mathbf{X})^{\top} \hat{\mathcal{H}}^{\mathcal{M}}(\mathbf{X}) g(\mathbf{X})
\end{equation*}
is minimized by
\begin{equation*}
	\hat{g}^{HS}(\mathbf{X}) = (\mathbf{W} + \lambda  \hat{\mathcal{H}}^{\mathcal{M}}(\mathbf{X}))^{-1} \mathbf{W} \mathbf{y},
\end{equation*}
where $\mathbf{W} = \diag (\mathbf{w})$.

\bigskip
There have been many previous studies concerning splines on more general spaces than subsets of $\mathbb{R}^n$. Splines on the sphere have been extensively researched from various perspectives \cite{bk1, bk2}, including but not limited to spherical harmonics \cite{freeden}, B\'ezier polynomials \cite{jekeli}, and thin-plate splines \cite{beatson}. Penalties on the sphere usually involve the norm of the Laplacian, so that the kernel consists of harmonic functions; however, this gives no penalty to functions of potentially high curvature, which we remedy by using the Hessian formulation.

One-dimensional cubic splines on curved surfaces were first studied in \cite{noakes}. Jupp and Kent \cite{jk} developed the so-called unroll-unwrap technique for fitting a one-dimensional path to points on a sphere, which has been extended to homogeneous spaces \cite{pauley}, and recently to general Riemannian manifolds \cite{krkim}, where a wide range of references on geodesic analysis can be found. Gu et al. \cite{gu} study splines on subdivision surfaces in $\mathbb{R}^3$, and develop the equivalent of the $B$-spline algorithm in this case. Hofer et al. \cite{hofer} also compute splines on parametric surfaces and triangle meshes of sampled points. Some theoretical results on interpolating splines on smooth and algebraic manifolds are given in Chapter 6 of \cite{bez}.
Our approach of a nonlinear generalization of thin-plate splines via estimation of the Hessian is new. For other methods of regression in the context of manifold learning, see also local linear regression \cite{cheng}, intrinsic polynomial regression \cite{hinkle}, and extrinsic local regression \cite{linl}.

Concerning nearest-neighbor methods: throughout the paper, we will work with a neighborhood $U_{\text{Euc},\mathbf{x}}$ (in $\mathbb{R}^n$) of an embedded data point $\mathbf{x}$ to model the neighborhood $U_{\mathbf{x}}$ in the original manifold $\mathcal{M}$. The neighborhood will consist of the closest points with respect to \textit{Euclidean} distance in $\mathbb{R}^n$, and the geometry of $U_{\text{Euc},\mathbf{x}}$ is implicitly taken to be a close approximation of the geometry of $U_{\mathbf{x}}$. This assumption is valid, as under mild regularity conditions, the embedded Euclidean distance approximates geodesic distance up to order $<3$ \cite{hein}:
\begin{equation*}
	d^{\mathcal{M}}(\mathbf{x},\mathbf{y}) = \norm{\mathbf{x}-\mathbf{y}} + O(\norm{\mathbf{x}-\mathbf{y}}^3), \quad \forall \mathbf{y}\in U_{\mathbf{x}}.
\end{equation*}

\section{Theory}

\subsection{Cubic Splines}

We first discuss the most basic problem of interpolating data in one dimension. For large datasets, polynomial interpolation of increasingly higher order does not always improve model accuracy, as exemplified by Runge's phenomenon. Instead, piecewise polynomial interpolants (splines) are used.

\begin{defn}
	Given data (`knots') $(x_i,y_i)$, $i=1,\cdots,N$ with $x_1<\cdots<x_N$, $f:[x_1,x_N] \rightarrow \mathbb{R}$ is an \textit{interpolating spline of degree $k$} if $f(x_i) = y_i$, $f$ restricted to each interval $[x_i,x_{i+1}]$ is a polynomial of order at most $k$, and $f\in C^{k-1}[x_1,x_N]$.
\end{defn}

In this paper, we will study the case $k=3$. Since each cubic polynomial has 4 coefficients, $f$ has a total of $4(N-1)$ parameters to be computed. The $C^2$ conditions at each knot yield $3(N-2)$ equations, while $f(x_i) = y_i$ give another $N$. Further requiring $f''(x_1) = f''(x_N) = 0$ determines the so-called natural cubic spline which extends linearly beyond the endpoint knots.

\bigskip
The cubic spline can be formulated as the solution of an optimization problem \cite{silver}. Namely, we wish to minimize the integral $\int (f'')^2$ over an interval $[a,b]$ containing $[x_1,x_N]$, over the function space $\mathcal{F} = \{f\in C^2[a,b]: f(x_i)=y_i,\; i=1,\cdots,N\}$. This penalizes how nonlinear or `curvy' the interpolating function $f$ is.

Let $f$ be the natural cubic spline and let $\tilde{f} \in \mathcal{F}$ be any other interpolating function. Setting $g = f-\tilde{f}$, a straightforward calculation shows:
\begin{equation*}
	\int_a^b f''(x)g''(x) dx = - \int_{x_1}^{x_N} f'''(x)g'(x) dx = - \sum_{j=1}^{N-1} f'''(x_j^+) (g(x_{j+1})-g(x_j)) = 0
\end{equation*}
so that $\displaystyle \int_a^b \tilde{f}''(x)^2 dx - \int_a^b f''(x)^2 dx = \int_a^b g''(x)^2 dx \geq 0$, which verifies our claim.

\bigskip
Various algorithms exist to efficiently compute one-dimensional natural cubic splines, such as \cite{aosp} or the method of B-spline bases \cite{htf}.

\subsection{Thin-Plate Splines}

Thin-plate splines (TPS) is a multidimensional spline-based data smoothing technique with applications in geometric design, introduced in \cite{duchon}. One fits a regression function $f$ to given data, whereby the bending energy of the surface defined by $f$ is minimized whilst also approximating the data as best as possible. TPS is closely related to the elastic maps technique \cite{elastic} in nonlinear dimensionality reduction.

We approach this problem in a rather informal manner, by using the calculus of variations. A more mathematically precise formulation requires the theory of reproducing kernel Hilbert spaces (RKHS), which we present and apply to our more general setting in Section 4.

In general, let $J(f)$ be a functional of the following integral form:
\begin{equation*}
	J(f) = \int_a^b L(x,f,f',f'') dx
\end{equation*}

Suppose $f$ is a minimizer of $J$. If we vary $f$ in the direction of another function $g$ such that $g$ and $g'$ evaluates to 0 at $a,b$, then $t \mapsto J(f+tg)$ must have a minimum at $t=0$. Thus the following formal derivative must be zero:
\begin{align*}
	\lim_{t\rightarrow 0} \frac{J(f+tg)-J(f)}{t} &= \int_a^b \lim_{t\rightarrow 0} \frac{1}{t} \left[ L(x,f+tg,f'+tg',f''+tg'')-L(x,f,f',f'') \right] dx \\
	&= \int_a^b \left( \frac{\partial L}{\partial f} g + \frac{\partial L}{\partial f'} g' + \frac{\partial L}{\partial f''} g'' \right) dx \\
	&= \int_a^b \left( \frac{\partial L}{\partial f} - \frac{d}{dx} \left( \frac{\partial L}{\partial f'} \right) + \frac{d^2}{dx^2} \left( \frac{\partial L}{\partial f''} \right) \right) g \; dx
\end{align*}
via integration by parts. Thus we obtain the analogue of the Euler-Lagrange equation for functionals involving order 2 derivatives:
\begin{equation*}
	\frac{\partial L}{\partial f} - \frac{d}{dx} \left( \frac{\partial L}{\partial f'} \right) + \frac{d^2}{dx^2} \left( \frac{\partial L}{\partial f''} \right) = 0
\end{equation*}

\noindent
For $L(x,f,f',f'') = (f'')^2$, this becomes $f^{(4)}=0$, which should be interpreted as holding almost everywhere: that is, $f$ is piecewise cubic.

\bigskip
The theory extends readily to the multivariate case. Let $f(\mathbf{x}) = f(x_1,\cdots,x_d)$ be a $d$-dimensional function with derivatives $\partial f / \partial x_j = f_{x_j}$, $\partial^2 f / \partial x_j \partial x_k = f_{x_jx_k}$ and so on. If the integrand $L$ depends on all $x_j$, $f$, $f_{x_j}$, and $f_{x_jx_k}$, the Euler-Lagrange equation becomes:

\begin{equation} \label{elmult}
	\frac{\partial L}{\partial f} - \sum_{j=1}^d \frac{d}{dx_j} \left( \frac{\partial L}{\partial f_{x_j}} \right) + \sum_{j=1}^d \sum_{k=1}^d \frac{d}{dx_j} \frac{d}{dx_k} \left( \frac{\partial L}{\partial f_{x_jx_k}} \right) = 0
\end{equation}

Now suppose we are given data $(\mathbf{x}_i,y_i)$, $i=1,\cdots,N$ with $\mathbf{x}_i \in \mathbb{R}^d$. The smoothed thin-plate spline is defined as the minimizer of:
\begin{equation*}
	\mathcal{L}^{TPS}(f;\lambda) = \sum_{i=1}^N (y_i - f(\mathbf{x}_i))^2 + \lambda \int_{\mathbb{R}^d} \sum_{i=1}^N \sum_{j=1}^N f_{x_ix_j}^2
\end{equation*}
where the first term is the squared fitting error, the second term measures the nonlinear distortion of $f$, and the `smoothing parameter' $\lambda$ controls the tradeoff between the two losses. When $\lambda \rightarrow 0$, the TPS interpolates the data perfectly; when $\lambda \rightarrow \infty$, the TPS approaches multivariate linear regression.

The TPS loss can be reformulated as
\begin{equation*}
	\mathcal{L}^{TPS}(f;\lambda) = \int_{\mathbb{R}^d} \left( \sum_{i=1}^N (y_i - f(\mathbf{x}))^2 \delta(\mathbf{x} - \mathbf{x}_i) + \lambda \sum_{i=1}^N \sum_{j=1}^N f_{x_ix_j}^2(\mathbf{x}) \right) d\mathbf{x}
\end{equation*}

\noindent
Applying equation (\ref{elmult}), we obtain, in terms of the Laplacian $\Delta = \sum_{j=1}^d \partial^2 / \partial x_j^2$:
\begin{equation} \label{biharmonic}
	\sum_{i=1}^N(f(\mathbf{x}) - y_i) \delta(\mathbf{x} - \mathbf{x}_i) + \lambda \Delta^2 f(\mathbf{x}) = 0
\end{equation}

\noindent
Equation (\ref{biharmonic}) can be solved using Green's functions of the biharmonic operator $\Delta^2$, that is, solutions of
\begin{equation*}
	\Delta^2 G(\mathbf{x},\mathbf{s}) = \delta(\mathbf{x}-\mathbf{s}) \quad \forall \mathbf{x}\in \mathbb{R}^d
\end{equation*}

\noindent
Taking $G$ to be of the form $G(\mathbf{x},\mathbf{s}) = \rho(|\mathbf{x}  -\mathbf{s}|)$, one can show:
\begin{equation*} \rho(r) = 
\begin{cases}
\frac{1}{8\pi} r^2 \log r & d=2 \\
- \frac{1}{8\pi^2} \log r & d=4 \\
\frac{1}{16\pi^{d/2}} \Gamma(\frac{d}{2}-2) r^{4-d} & \text{else}
\end{cases}
\end{equation*}

\noindent
Note that $\rho$ has an undesirable singularity at 0 for dimensions $d\geq 4$, so modified radial basis functions are usually chosen as to be bounded.

The solution to (\ref{biharmonic}) which minimizes the TPS loss may be expressed as a combination of Green's functions at the data $\mathbf{x}_i$ and a linear component as in \cite{tps},
\begin{equation*}
	\hat{f}(\mathbf{x}) = \sum_{i=1}^N a_i G(\mathbf{x},\mathbf{x}_i) + b_0 + \sum_{j=1}^d b_j x_j
\end{equation*}
which yields the following equation (in the distributional sense),
\begin{equation*}
	\sum_{i=1}^N (\hat{f}(\mathbf{x}) - y_i + \lambda a_i) \delta(\mathbf{x} - \mathbf{x}_i) = 0
\end{equation*}

\noindent
Thus we must have $\hat{f}(\mathbf{x}_i) - y_i + \lambda a_i = 0$ for all $i$. Let $\mathbf{G}$ be the $N\times N$ matrix with entries $G(\mathbf{x}_i, \mathbf{x}_j)$, and $\mathbf{X}$ the $N\times (d+1)$ design matrix with rows $(1\;\; \mathbf{x}_i^{\top})$.
Let $\mathbf{a}$ and $\mathbf{b}$ denote the column $N-$ and $(d+1)-$ vectors formed by $a_i$ and $b_j$, respectively. The conditions can then be rewritten as:
\begin{equation*}
	\mathbf{G}\mathbf{a} + \mathbf{X}\mathbf{b} - \mathbf{y} + \lambda \mathbf{a} = 0
\end{equation*}
with the additional orthogonality condition $\mathbf{X}^{\top} \mathbf{a} = 0$, see \cite{tps}. This is solved by
\begin{equation*}
	\hat{\mathbf{b}} = \left( \mathbf{X}^{\top} (\mathbf{G} + \lambda \mathbf{I}_N)^{-1} \mathbf{X} \right)^{-1} \mathbf{X}^{\top} (\mathbf{G} + \lambda \mathbf{I}_N)^{-1} \mathbf{y}, \quad \hat{\mathbf{a}} = (\mathbf{G} + \lambda \mathbf{I}_N)^{-1} (\mathbf{y} - \mathbf{X} \hat{\mathbf{b}})
\end{equation*}

The smoothing parameter $\lambda$ is usually chosen by cross-validation criteria. Writing $\hat{f}^{(-i)}$ as the TPS fitted to the same data but with the $i$th observation $(\mathbf{x}_i,y_i)$ deleted, the CV index is defined as a sum of squared errors (possibly weighted):
\begin{equation*}
	CV(\lambda) = \frac{1}{N}\sum_{i=1}^N \left(y_i - \hat{f}^{(-i)}(\mathbf{x}_i)\right)^2
\end{equation*}
and $\lambda$ is chosen as the value minimizing this index. See \cite{aosp} for a related discussion of cross-validation.

\subsection{Hessian Eigenmaps}

In the setting of manifold learning, we assume that our data, distributed in a high-dimensional Euclidean \textit{feature space} $\mathbb{R}^n$, actually lies on or near a lower-dimensional Riemannian submanifold $\mathcal{M}$. We aim to recover $\mathcal{M}$ up to some class of interest, the embedding map $\psi$ into $\mathbb{R}^n$, and the corresponding representations $\psi^{-1}(\mathbf{x}_i)$ of the data.

Hessian Eigenmaps, developed by Donoho and Grimes \cite{hlle}, is a popular nonlinear dimensionality reduction method to achieve this task. The algorithm supposes that the embedded (\textit{articulation}) submanifold is locally isometric to an open, connected subset $\Theta \subseteq \mathbb{R}^d$ (the \textit{parameter space}). The locally isometric embedding is denoted by $\psi: \Theta \rightarrow \psi(\Theta) = \mathcal{M} \subset \mathbb{R}^n$. Given data $\mathbf{x}_i \in \mathbb{R}^n$, $i=1,\cdots,N$, the algorithm retrieves the original coordinates $(x^j \circ \psi^{-1})(\mathbf{x}_i)$, $j=1,\cdots,d$.

Note these are weaker conditions compared to other widely-used algorithms such as Isomap or LLE, which require $\Theta$ to be convex and $\psi$ to be a global isometry \cite{mlta}.

\bigskip
Recall that for a Riemannian manifold $\mathcal{M}$ with Levi-Civita connection $\nabla$, the Hessian tensor of a smooth function $f:\mathcal{M} \rightarrow \mathbb{R}$ is defined as $\Hess_f := \nabla \nabla f$, a section of the bundle $T^*\mathcal{M} \otimes T^*\mathcal{M}$. In local coordinates $\{x^j\}$ around $p\in\mathcal{M}$, the Hessian is expressed in Einstein notation as:
\begin{equation*}
	\Hess_f(p) = \left( \frac{\partial^2 f}{\partial x^i \partial x^j}(p) - \Gamma^k_{ij} \frac{\partial f}{\partial x^k}(p) \right) dx^i \otimes dx^j
\end{equation*}
where $\Gamma_{ij}^k$ are the Christoffel symbols of $\nabla$. For our problem of interest, since $\mathcal{M}$ is locally isometric to Euclidean space, by Theorema Egregium we have $\Gamma_{ij}^k = 0$.

\bigskip
Let $p\in\mathcal{M}$ and $U_p$ a suitable neighborhood of $p$. There are different ways to define local coordinates on $U_p$ for which the Hessian may be computed:
\begin{enumerate}
	\item \textit{Isometric} coordinates: by pushing forward the Euclidean coordinates $x^j$ on $\Theta$, $U_p$ inherits the coordinate system $\theta^j(q) = (x^j \circ \psi^{-1})(q)$. The associated Hessian matrix is denoted by $\Hess_f^{\tan}(p)$.
	
	\item \textit{Normal} coordinates: the exponential map $\exp_p: T_p\mathcal{M} \supset V \rightarrow \mathcal{M}$ provides a diffeomorphism between $U_p$ and a subset of the tangent space. Viewing $T_p\mathcal{M}$ as an affine subspace of $\mathbb{R}^n$, we may define any set of orthonormal coordinates (the choice will not matter) on $T_p\mathcal{M}$, which in turn defines coordinates on $U_p$. The associated Hessian matrix is denoted by $\Hess_f^{\geo}(p)$.
	
	\item \textit{Tangent} coordinates: Again viewing $T_p\mathcal{M}$ as sitting inside $\mathbb{R}^n$, the linear projection map $\pi_p: U_p \rightarrow T_p\mathcal{M}$ provides another diffeomorphism between $U_p$ and a subset of $T_p\mathcal{M}$. The associated Hessian matrix is denoted by $\Hess_f^{\tan}(p)$. Note that only this definition can provide a tractable estimation scheme.
\end{enumerate}

We are concerned only with the Frobenius norms of the above matrices, which are invariant with respect to orthogonal coordinate transformations and thus well-defined. It can be shown that
\begin{equation*}
	\lVert{\Hess_f^{\tan}(p)\rVert}_F = \lVert{\Hess_f^{\geo}(p)\rVert}_F = \lVert{\Hess_f^{\iso}(p)\rVert}_F
\end{equation*}

Finally, let $d\mu$ be a probability measure on $\mathcal{M}$ with positive density everywhere on the interior of $\mathcal{M}$ from which our data is sampled. Our object of interest is the following integral
\begin{equation*}
	\mathcal{H}^{\mathcal{M}}(f):=\int_{\mathcal{M}}\nolimits \lVert{\Hess_f^{\tan}(p)\rVert}_F^2 d\mu(p),
\end{equation*}
which is a quadratic form on the Sobolev space $W_2^2(\mathcal{M})$ and gives a measure of the average `curviness' of $f$ over $\mathcal{M}$. Considering the isometric formulation, $\mathcal{H}^{\mathcal{M}}$ clearly evaluates to zero for functions linear in the $x^j$. In fact, the converse statement is also true.

\begin{thm} \label{nullspace}
	The quadratic form $\mathcal{H}^{\mathcal{M}}$ on $W_2^2(\mathcal{M})$ has a $(d+1)$-dimensional null-space $\ker \mathcal{H}^{\mathcal{M}}$ spanned by the constant function and the isometric coordinates $\theta^i$.
\end{thm}
\begin{proof}
	(Sketch.) The statement is first proved for $C^2$ functions on Euclidean space, then the natural pullback from $\mathcal{M}$ to $\Theta$ is used to extend the result to $\mathcal{M}$ by the equality
	$\lVert{\Hess_f^{\tan}(p)\rVert}_F = \lVert{\Hess_f^{\iso}(p)\rVert}_F$.
\end{proof}

This allows us to retrieve the original parametric coordinates $\theta^j(\mathbf{x}_i)$ of the embedded data by constructing a discrete estimator $\hat{\mathcal{H}}^{\mathcal{M}}$ and studying its eigenstructure. The estimator will be an $N\times N$ matrix, depending on the data $\mathbf{X} = (\mathbf{x}_1 \cdots \mathbf{x}_N)^{\top}$, with the property that for any $C^2$ function $f: \mathcal{M} \rightarrow \mathbb{R}$, one may approximate
\begin{equation} \label{happrox}
	\mathcal{H}^{\mathcal{M}}(f) \approx f(\mathbf{X})^{\top} \hat{\mathcal{H}}^{\mathcal{M}}(\mathbf{X}) f(\mathbf{X})
\end{equation}

\bigskip
\textbf{Deriving the Hessian estimator}

Denote by $N(\mathbf{x}_i)$ the ordered set of the $K(\geq d)$ data points closest to $\mathbf{x}_i$ with respect to the distance in $\mathbb{R}^n$. This serves as a proxy for the open neighborhood $U_{\mathbf{x}_i}$. We work with the idea that the local tangent space $T_{\mathbf{x}_i} \mathcal{M} \subset \mathbb{R}^n$ is best represented by the affine subspace $\hat{T}(\mathbf{x}_i)$ least-squares fitted to $N(\mathbf{x}_i)$. Applying Principal Component Analysis (PCA) to the subset $N(\mathbf{x}_i)$ gives a direct estimate of the local tangent coordinates. That is, the eigenvectors $\mathbf{u}_1^{(i)}, \cdots, \mathbf{u}_d^{(i)}$ corresponding to the largest $d$ eigenvalues of the $K \times K$ Gram matrix $(\mathbf{G}^{(i)})_{jk} = (\mathbf{x}_j - \mathbf{x}_i)^{\top} (\mathbf{x}_k - \mathbf{x}_i)$ give the orthogonally projected coordinates from $N(\mathbf{x}_i)$ to $\hat{T}(\mathbf{x}_i)$. For noisy data, one may apply a robust version of PCA.

At each $\mathbf{x}_i$, we find a $({}_d\mathrm{H}_2 \times K)$-matrix $\mathbf{H}^{(i)}$ that approximates the tangent coordinate-based Hessian in the sense that for any $f\in C^2(\mathcal{M})$,
\begin{equation*}
	\mathbf{H}^{(i)}f^{(i)},\quad \text{where } f^{(i)} = (\cdots,f(\mathbf{x}_j),\cdots)^{\top}, \quad \mathbf{x}_j\in N(\mathbf{x}_i)
\end{equation*}
is a length $d(d+1)/2$ vector whose entries approximate each
\begin{equation*}
	(\Hess_f^{\tan}(\mathbf{x}_i))_{\alpha\beta} = \frac{\partial^2 f}{\partial x^{\alpha} \partial x^{\beta}} (\mathbf{x}_i), \quad 1\leq \alpha \leq \beta \leq d
\end{equation*}

In particular, for each row index $(\alpha\beta)$ corresponding to the index pair $\alpha, \beta$, we want the following estimation scheme to hold:
\begin{equation}\label{scheme}
	\frac{\partial^2 f}{\partial x^{\alpha} \partial x^{\beta}} (\mathbf{x}_i) \approx \sum_j  \mathbf{H}_{(\alpha \beta) ,\, j}^{(i)} f(\mathbf{x}_j) \quad \forall f\in C^{\infty}(\mathcal{M})
\end{equation}

Write $\epsilon_{j}^{(i)}:=\mathbf{x}_j-\mathbf{x}_i$. Substituting $f(\mathbf{x}_j)$ by its second order Taylor expansion $f(\mathbf{x}_i)+ \sum_k \frac{\partial f}{\partial x^k}(\mathbf{x}_i)\, \epsilon_{j,k}^{(i)} + \frac{1}{2} \sum_{k,\ell} \frac{\partial^2 f}{\partial x^k \partial x^{\ell}}(\mathbf{x}_i)\, \epsilon_{j,k}^{(i)} \epsilon_{j,\ell}^{(i)} \,$ above gives
\begin{align*}
	\frac{\partial^2 f}{\partial x^{\alpha} \partial x^{\beta}} (\mathbf{x}_i) &\simeq \left( \sum_j \mathbf{H}_{(\alpha \beta) ,\, j}^{(i)}\right) f(\mathbf{x}_i) + \sum_k \frac{\partial f}{\partial x^k}(\mathbf{x}_i) \left( \sum_j \mathbf{H}_{(\alpha \beta) ,\, j}^{(i)} \epsilon_{j,k}^{(i)} \right) \\ & + \frac{1}{2} \sum_{k, \ell} \frac{\partial^2 f}{\partial x^k \partial x^{\ell}}(\mathbf{x}_i) \left( \sum_j \mathbf{H}_{(\alpha \beta) ,\, j}^{(i)} \epsilon_{j,k}^{(i)} \epsilon_{j,\ell}^{(i)} \right)
\end{align*}

Since each tangent $k$-coordinate $\epsilon_{*,k}^{(i)}$ is collected into $\mathbf{u}_k^{(i)}$, we thus require the relations:
\begin{equation*}
	\mathbf{H}_{(\alpha \beta)}^{(i)\,\,\top} 1_K = 0, \quad \mathbf{H}_{(\alpha \beta)}^{(i)\,\,\top} \mathbf{u}_k^{(i)} = 0, \quad \mathbf{H}_{(\alpha \beta)}^{(i)\,\,\top} \left( \mathbf{u}_k^{(i)}*\mathbf{u}_{\ell}^{(i)} \right) = 2\delta_{k,\ell}^{\alpha, \beta}
\end{equation*}
where $1_K$ is the vector consisting of $K$ ones, and $*$ denotes entrywise multiplication of vectors of equal length. This system may be solved by performing Gram-Schmidt orthogonalization on the following $K \times (1+d+d(d+1)/2)$-matrix:
\begin{equation*}
	\left( 1_K \,\Big|\, \mathbf{u}_k^{(i)} \cdots \,\Big|\, \mathbf{u}_k^{(i)}*\mathbf{u}_{\ell}^{(i)} \cdots \right)
\end{equation*}
and taking the last $d(d+1)/2$ columns as the corresponding rows $\mathbf{H}_{(\alpha \beta)}^{(i)\,\,\top}$, suitably normalized. Finally, we construct $\hat{\mathcal{H}}^{\mathcal{M}}(\mathbf{X})$ via a form of contraction,
\begin{equation*}
	\big( \hat{\mathcal{H}}^{\mathcal{M}}(\mathbf{X})\big)_{jm} := \frac{1}{N} \sum_i \sum_{\alpha, \beta} \mathbf{H}_{(\alpha \beta),\,j}^{(i)} \mathbf{H}_{(\alpha \beta),\,m}^{(i)}
\end{equation*}

It is straightforward to check that $\hat{\mathcal{H}}^{\mathcal{M}}(\mathbf{X})$ has our desired property.
\begin{align*}
	f(\mathbf{X})^{\top} \hat{\mathcal{H}}^{\mathcal{M}}(\mathbf{X}) f(\mathbf{X}) &= \frac{1}{N} \sum_{i,j,m} \sum_{\alpha, \beta} f(\mathbf{x}_j) f(\mathbf{x}_m) \mathbf{H}_{(\alpha \beta),\,j}^{(i)} \mathbf{H}_{(\alpha \beta),\,m}^{(i)} \\
	&\approx \frac{1}{N} \sum_i \sum_{\alpha, \beta} \left( \frac{\partial^2 f}{\partial x^{\alpha} \partial x^{\beta}} (\mathbf{x}_i) \right)^2 \\
	&= \frac{1}{N} \sum_i \lVert{\Hess_f^{\tan}(\mathbf{x}_i)\rVert}_F^2 \\
	&\approx \int_{\mathcal{M}}\nolimits \lVert{\Hess_f^{\tan}(p)\rVert}_F^2 d\mu(p)
\end{align*}

Due to the approximations involved, the quadratic form represented by $\hat{\mathcal{H}}^{\mathcal{M}}(\mathbf{X})$ no longer satisfies the statement of Theorem \ref{nullspace} exactly. Nevertheless, the original coordinate functions $\theta^j$ must still give small values, assuming the input functions are normalized so that $\norm{f(\mathbf{X})}=1$. Thus, computing its spectral decomposition and taking the eigenvectors corresponding to the $d$ smallest eigenvalues should yield approximate values for $\theta^j(\mathbf{x}_i)$. Note that the algorithm requires pre-chosen values for $d$ and $K$, unlike algorithms such as Isomap where $d$ may be chosen by a scree test after the spectral decomposition step.

\section{Smooth Manifold Splines}

Let $\mathcal{M}$ be a connected $d$-dimensional flat Riemannian manifold, not necessarily a subset of $\mathbb{R}^d$.
Suppose we are given observations $(p_i,y_i)$, $i=1,\cdots,N$, where $p_i\in\mathcal{M}$ and $y_i\in \mathbb{R}$. The $p_i$ are sampled from a probability density $\mu$ on $\mathcal{M}$, and the observations follow the model $y_i=f(p_i) + \epsilon_i$ where the error terms are i.i.d. $N(0,\sigma^2)$. We define the \textit{smoothing spline estimator} $\hat{f}^{SS}$ for $f$ with a smoothing parameter $\lambda \geq 0$ as
\begin{equation} \label{ssloss2}
	\hat{f}^{SS} = \argmin_{f\in W_2^2(\mathcal{M})} \mathcal{L}^{SS}(f;\lambda), \quad \mathcal{L}^{SS}(f;\lambda) = \sum_{i=1}^N (y_i-f(p_i))^2 + \lambda\cdot \mathcal{H}^{\mathcal{M}}(f)
\end{equation}

Before evaluating the model, we first discuss the global penalty term $\mathcal{H}^{\mathcal{M}}(f)$. What does it mean for $\mathcal{H}^{\mathcal{M}}(f)$ to be zero, that is for the Hessian to vanish everywhere, or more generally have small value? While we mentioned that $\mathcal{H}^{\mathcal{M}}$ measures a certain `curviness' of the function $f$, and its null-space is characterized by the linear functions if $\mathcal{M} \subset \mathbb{R}^d$, the weaker condition of flatness is worth studying in more detail.

For example, consider a function with vanishing Hessian on the flat cylinder $\mathcal{M} = \mathbb{R} \times (\mathbb{R}/\mathbb{Z})$. Restricting to the subset $\mathbb{R} \times (0,1)$, $f$ must be linear of the form $ax+by+c$ by Theorem \ref{nullspace}. Continuity $f(\,\cdot\,,0)=f(\,\cdot\,,1)$ then forces $b=0$, so $\ker \mathcal{H}^{\mathbb{R}/\mathbb{Z}}$ is spanned by the constants and $(x,y) \mapsto x$. More generally, any flat complete manifold has Euclidean space as universal cover, which permits their classification into quotient spaces with respect to actions by Bieberbach groups (torsionfree crystallographic groups) \cite{bbb}. Restricting to a fundamental domain $\mathcal{F} \subset \mathbb{R}^d$, any function with vanishing Hessian must be linear in $1, x^1,\cdots,x^d$. Continuity conditions at the boundary $\partial \mathcal{F}$ `kill off' or introduce relations within these generators, one per each pair of identified piecewise smooth components, reducing the dimension of $\mathcal{N}^{\mathcal{M}} := \ker \mathcal{H}^{\mathcal{M}}$.

\begin{prop}\label{affine}
	$\Hess_f=0$ iff for every geodesic $\gamma$, $f(\gamma(t))$ is an affine function of $t$.
\end{prop}
\begin{proof}
	(Sketch.) For any Riemannian manifold $\mathcal{M}$ with a torsionfree connection $\nabla$, we can compute
	\begin{equation*}
		(f\circ \gamma)\ddot{\;}(t) = \Hess_f(\gamma(t)) (\dot{\gamma}(t), \dot{\gamma}(t)) + df_{\gamma(t)} \left( \frac{D\dot{\gamma}}{dt}(t) \right)
	\end{equation*}
	so that $(f\circ \gamma)\ddot{\;}$ becomes zero if the Hessian vanishes.
\end{proof}

Thus, $f$ is completely determined by the information $f(p)$ and $d_pf$ at any single point $p\in\mathcal{M}$, showing that $\nu:=\dim \mathcal{N}^{\mathcal{M}} \leq d+1$. However, this information may fail to be consistent, which is measured by the fundamental group $\pi_1(\mathcal{M})$. For example, if $\partial\mathcal{M}=\varnothing$ we have by \cite{morse}:
\begin{thm}
	Let $\mathcal{M}$ be a complete Riemannian manifold with non-positive sectional curvature and fix $p\in \mathcal{M}$. Then every homotopy class in $\pi_1(\mathcal{M},p)$ has a unique geodesic representative $\gamma: [0,1] \rightarrow \mathcal{M}$ with $\gamma(0)=\gamma(1)=p$.
\end{thm}

In particular, due to Proposition \ref{affine}, $f\circ \gamma$ must be constant if $\Hess_f=0$ and $\gamma$ is a geodesic representative of an element of $\pi_1(\mathcal{M})$. This forces $d_pf=0$ when restricted to the subspace spanned by $\{\dot{\gamma}(0): [\gamma] \in \pi_1(\mathcal{M})\}$, and therefore $\nu \leq d+1- \dim \pi_1(\mathcal{M})$.

\bigskip
In any case, we may find a finite unisolvent basis $\{\zeta^j\}_{j=1}^{\nu}$ of $\mathcal{N}^{\mathcal{M}}$, analogous to the linear functions $x_j$ on a subspace of $\mathbb{R}^d$. In practice, as in the example above, finding $\zeta^j$ in most cases will simply be a matter of checking which $x_j$ are ruled out in case $\mathcal{M}$ has nontrivial fundamental group. In the next Section, we prove the following central result:

\begin{thm}\label{mainthm}
	If $\mathcal{M}$ is a flat, compact Riemannian manifold with $\partial \mathcal{M} = \varnothing$, the solution of the manifold smoothing spline optimization problem
	\begin{equation*}
		\hat{f}^{SS} = \argmin\sum_{i=1}^N (y_i-f(p_i))^2 + \lambda\cdot \mathcal{H}^{\mathcal{M}}(f)
	\end{equation*}
	formulated in the Sobolev space $W^{2,2}(\mathcal{M})$ exists, is unique, and admits a representation of the form
	\begin{equation*} 
		\hat{f}(p) = \sum_{i=1}^N a_i G(p,p_i) +  \sum_{j=1}^{\nu} b_j \zeta^j(p)
	\end{equation*}
where $G$ is a Green's function for the biharmonic operator $\Delta_{\mathcal{M}}^2$. If more generally $\partial\mathcal{M} \neq \varnothing$, the same result holds when Neumann boundary conditions (corresponding to linearity) are imposed,
\begin{equation*}
	W_{\partial}^{2,2}(\mathcal{M}):= \{f\in W^{2,2}(\mathcal{M}): Nf = 0 \text{ on }\partial\mathcal{M}\} \oplus \mathcal{N}^{\mathcal{M}} \subseteq W^{2,2}(\mathcal{M})
\end{equation*}
where $N$ is the outwards-oriented unit normal vector field on $\partial\mathcal{M}$.
\end{thm}

We are moving to the setting of Sobolev spaces in order to utilize the machinery of Hilbert spaces; derivatives are to be understood in a suitable weak sense. The necessary results from the theory of partial differential equations on manifolds are presented here. For a detailed exposition, see Chapter 5 of \cite{taylor}. Let $\mathcal{M}$ be a compact oriented Riemannian manifold, possibly with boundary. Let $\iota: \partial\mathcal{M} \rightarrow \mathcal{M}$ be the canonical embedding and $d\omega$ the induced volume element on $\partial\mathcal{M}$. Recall the Sobolev space $H_0^s(\mathcal{M})$ is defined as the closure of $C_c^{\infty}(\mathcal{M})$ in $H^s(\mathcal{M}) = W^{s,2}(\mathcal{M})$ with respect to the Sobolev norm.
\begin{prop}
	The Laplace-Beltrami operator $\Delta_{\mathcal{M}}$ extends to a well-defined elliptic partial differential operator of order 2 on $H_0^1(\mathcal{M})$, yielding an isomorphism
	\begin{equation*}
		\Delta_{\mathcal{M}}: H_0^1(\mathcal{M}) \rightarrow H^{-1}(\mathcal{M})
	\end{equation*}
	The inverse operator $T_{\mathcal{M}}$ restricted to $L^2(\mathcal{M})$ is compact and negative self-adjoint, by Rellich's theorem.
\end{prop}

Before we proceed, the issue of boundary conditions in the current setting must also be discussed. Since elements of Sobolev spaces are \textit{equivalence classes} of functions with respect to the volume measure on $\mathcal{M}$, and $\partial\mathcal{M}$ always has measure zero, familiar expressions such as $f|_{\partial\mathcal{M}}$ or $Nf$ are a priori not well-defined. However, such constructions are still possible due to the
\begin{thm}[Trace Theorem] \label{trace}
	Let $\mathcal{M}$ be a Riemannian manifold with smooth boundary $\partial\mathcal{M}$. Then for $1\leq p<\infty$, there exists a unique bounded linear operator $\gamma: W^{1,p}(\mathcal{M}) \rightarrow L^p(\partial\mathcal{M})$ extending the classical trace,
	\begin{equation*}
		\gamma u = u|_{\partial\mathcal{M}} \quad \text{for } u\in W^{1,p}(\mathcal{M}) \cap C(\partial\mathcal{M})
	\end{equation*}
	
\end{thm}
The space $H_0^1(\mathcal{M})$ can then be taken as the subspace of $H^1(\mathcal{M})$ with vanishing trace, and the Dirac delta distribution $\delta_q$ for $q\in\mathcal{M}$ can be identified with the bounded linear functional $f\mapsto f(q)$ on $H_0^1(\mathcal{M})$, i.e. an element of the dual space $H^{-1}(\mathcal{M})$.

\begin{thm}[Poisson's equation with Dirichlet boundary conditions]
	For $f\in H^{-1}(\mathcal{M})$ and $g\in C^{\infty}(\partial\mathcal{M})$, the boundary problem
	\begin{equation*}
		\Delta_{\mathcal{M}} u=f \text{ on } \mathcal{M},\quad \gamma u=g \text{ on } \partial\mathcal{M}
	\end{equation*}
has a unique solution in $H^1(\mathcal{M})$.
\end{thm}
\begin{proof}
	Construct $\tilde{g} \in C^{\infty}(\mathcal{M})$ so that $\tilde{g}|_{\partial\mathcal{M}} = g$. Then the above problem is equivalent to
	\begin{equation*}
		\Delta_{\mathcal{M}} v=f- \Delta_{\mathcal{M}}\tilde{g},\quad \gamma v=0
	\end{equation*}
for $v=u-\tilde{g}$. The solution is then found as $v = T_{\mathcal{M}}(f- \Delta_{\mathcal{M}}\tilde{g}) \in H_0^1(\mathcal{M})$, and uniqueness is clear.
\end{proof}

Since $W^{2,2}(\mathcal{M}) \subset W^{1,2}(\mathcal{M})$, it makes sense to talk about elements of $W^{2,2}(\mathcal{M})$ with vanishing trace. Also, for $f\in W^{2,2}(\mathcal{M})$, we have $\nabla f \in W^{1,2}(\mathcal{M})$ and
\begin{equation*}
	\lVert Nf \rVert_{L^2(\partial\mathcal{M})} \leq \lVert \gamma \, \nabla f \rVert_{L^2(\partial\mathcal{M})} \leq C\lVert \nabla f \rVert_{W^{1,2}(\mathcal{M})} <\infty
\end{equation*}
so that Neumann boundary conditions $Nf = g\in L^2(\partial\mathcal{M})$ are also well-defined. More generally, the following result can be proved for Neumann conditions (but we only require the case where $g$ is constant):
\begin{thm}[Poisson's equation with Neumann boundary conditions]
	For $f\in H^k(\mathcal{M})$ and $g\in H^{k+1/2}(\partial\mathcal{M})$, the boundary problem
	\begin{equation*}
		\Delta_{\mathcal{M}} u=f \text{ on } \mathcal{M},\quad N u=g \text{ on } \partial\mathcal{M}
	\end{equation*}
	has a unique solution in $H^{k+2}(\mathcal{M})$ if and only if
	\begin{equation*}
		\int_{\mathcal{M}} f d\mu = \int_{\partial\mathcal{M}} g d\omega
	\end{equation*}
	holds.
\end{thm}
The above constraint is a consequence of the divergence theorem for manifolds.

\bigskip
In the rest of this Section, we mirror the informal approach in Section 2 to motivate the use of \textit{biharmonic} Green's functions. As with TPS, we write the penalty (\ref{ssloss2}) in integral form:
\begin{equation*}
	\mathcal{L}^{SS}(f;\lambda) = \int_{\mathcal{M}} \left( \sum_{i=1}^N (y_i - f(p))^2 \delta_{p_i}(p) + \lambda \sum_{i=1}^N \sum_{j=1}^N f_{x_ix_j}^2(p) \right) d\mu(p)
\end{equation*}
where $\delta_q$ is the Dirac delta distribution on $\mathcal{M}$ such that the formal integral $\int_{\mathcal{M}} g(p) \delta_q(p) d\mu(p) = g(q)$ for all suitable functions $g$. The calculus of variations also works in the setting of Riemannian geometry, where the equations are to be understood as holding on every local orthonormal coordinate patch $\{x^j\}$. Applying the Euler-Lagrange equation (\ref{elmult}) gives
\begin{equation*}
	\sum_{i=1}^N (f(p)-y_i) \delta_{p_i}(p) + \lambda\Delta_{\mathcal{M}}^2 f(p) = 0
\end{equation*}

\bigskip
Now, the solution $\hat{f}^{SS}$ may be expressed in a finite-dimensional subspace as a linear combination of Green's functions $G(p,p_i)$ and an orthogonal term in $\ker \mathcal{H}^{\mathcal{M}}$, transforming the problem into a simple linear equation of the coefficients.
\begin{equation*}
	f(p) = \sum_{i=1}^N a_i G(p,p_i) +  \sum_{j=1}^{\nu} b_j \zeta^j(p)
\end{equation*}

Similarly to TPS, we solve the system $f(p_i) - y_i + \lambda a_i = 0$. Letting $\mathbf{G}_{ij} = G(p_i, p_j)$, and $\zeta(\mathbf{P})$ the $\nu \times N$ design matrix with $\mathbf{Z}_{ij} = \zeta^i(p_j)$, we have:
\begin{equation*}
	\mathbf{G}\mathbf{a} + \zeta(\mathbf{P})^{\top} \mathbf{b} - \mathbf{y} + \lambda \mathbf{a} = 0
\end{equation*}
with the orthogonality condition $\zeta(\mathbf{P}) \mathbf{a} = 0$. The system is solved by
\begin{equation*}
	\hat{\mathbf{b}} = \left( \zeta(\mathbf{P}) (\mathbf{G} + \lambda \mathbf{I}_N)^{-1} \zeta(\mathbf{P})^{\top} \right)^{-1} \zeta(\mathbf{P}) (\mathbf{G} + \lambda \mathbf{I}_N)^{-1} \mathbf{y}, \quad \hat{\mathbf{a}} = (\mathbf{G} + \lambda \mathbf{I}_N)^{-1} (\mathbf{y} - \zeta(\mathbf{P})^{\top} \hat{\mathbf{b}})
\end{equation*}

\section{RKHS on Manifolds}

\subsection{Theory of RKHS}
In order to prove Theorem \ref{mainthm}, we first present the theory of reproducing kernel Hilbert spaces, which provide powerful tools to resolve optimization problems of certain classes of functions.

A Hilbert space $\mathscr{H}$ of real-valued functions on a space $\mathscr{E}$ with inner product $\langle \cdot,\cdot \rangle$ is called a \textit{reproducing kernel Hilbert space (RKHS)} if all evaluation functionals $\phi_t : f \mapsto f(t)$ are continuous. In any normed vector space of functions $\mathscr{H}$, it is easy to see that:

\begin{prop}
	The following conditions are equivalent:
	
	\begin{enumerate}
		\item The evaluation functionals $\phi_t$ are continuous for all $t\in\mathscr{E}$.
		
		\item For $f, \{f_n\}_{n=1}^{\infty} \in \mathscr{H}$ satisfying $\norm{f_n-f} \rightarrow 0$, then $f_n(t) \rightarrow f(t)$ for all $t\in\mathscr{E}$.
		
		\item For every $t\in\mathscr{E}$ there exists $C_t$ such that $|f(t)| \leq C_t \norm{f}$ for all $f\in\mathscr{H}$.
	\end{enumerate}
\end{prop}

By the Riesz Representation Theorem, we are able to construct a kernel $K: \mathscr{E}\times\mathscr{E} \rightarrow \mathbb{R}$ so that $K(\cdot, t)$ is the unique function satisfying
\begin{equation}\label{reprodprop}
	f(t) = \phi_t(f) = \langle K(\cdot, t), f\rangle.
\end{equation}
$K$ is called the \textit{reproducing kernel} generating $\mathscr{H}$, and equation (\ref{reprodprop}) is called the \textit{reproducing property}. In particular, $K(t,s) = \langle K(\cdot, t), K(\cdot, s) \rangle$ shows that $K$ is symmetric.

We have the following \textit{projection principle} for RKHS \cite{gu}.
\begin{prop}
	Let $\mathscr{H}$ be a RKHS with reproducing kernel $K$, and let $\mathscr{H}_0$ be a subspace of $\mathscr{H}$ so that $\mathscr{H}$ is decomposed as $\mathscr{H}_0 \oplus \mathscr{H}_0^{\perp}$. Then the associated kernel decomposition is
	\begin{equation*}
		K(\cdot,t) = K_0(\cdot,t) + K_1(\cdot,t), \quad K_0(\cdot,t)\in \mathscr{H}_0, \;\; K_1(\cdot,t)\in \mathscr{H}_0^{\perp}
	\end{equation*}
iff $K_0$ and $K_1$ are the reproducing kernels for $\mathscr{H}_0$ and $\mathscr{H}_0^{\perp}$, respectively.
\end{prop}

We now present a general methodology for solving a wide range of penalized function optimization problems, see e.g. \cite{alm, wahba}. Let $\mathscr{S}$ be a vector space of functions on $\mathscr{E}$ and let $J$ be a nonnegative penalty functional on $\mathscr{S}$. Suppose $J$ has null set $\mathscr{N}$ which is a linear subspace of $\mathscr{S}$, and $J(f+\eta) = J(f)$ for all $\eta\in \mathscr{N}$. Suppose also that there exists some RKHS $\mathscr{H} \subset \mathscr{S}$ so that $J(f) = \langle f,f \rangle_{\mathscr{H}}$ for $f\in\mathscr{H}$. Note that this implies $\mathscr{N} \cap \mathscr{H} = 0$. Finally, consider a finite-dimensional subspace $\mathscr{N}_0 \subset \mathscr{N}$, generated by a basis $\{\zeta_j: j=1,\cdots,\nu\}$. (For our problem, $\mathscr{N}_0= \mathscr{N}$.) Define the vector space
\begin{equation*}
	\mathscr{T} = \mathscr{H} \oplus \mathscr{N}_0
\end{equation*}
Now predictors $x_i\in \mathscr{E}$ and responses $y_i\in\mathbb{R}$, $i=1,\cdots,N$ are given, following the model $y=f(x)+\epsilon$ for $f\in\mathscr{S}$ and an error term $E(\epsilon)=0$. Our goal is to optimize the penalized square error,
\begin{equation}\label{rkhsloss}
	\mathcal{L}(f) =\sum_{i=1}^N (y_i-f(x_i))^2 + \lambda\cdot J(f)
\end{equation}
for some $\lambda>0$. Then the following celebrated result holds:
\begin{thm}[Representer Theorem] \label{representer}
	Any minimizer $\hat{f}\in\mathscr{T}$ of (\ref{rkhsloss}) is of the form
	\begin{equation*}
		\hat{f}(x) = \sum_{i=1}^N \alpha_i K(x_i,x) + \sum_{j=1}^{\nu} \beta_j \zeta_j(x)
	\end{equation*}
\end{thm}

We give a rough proof; for details, see \cite{wahba}.

Decompose $\hat{f} = h+\eta$ with $h\in\mathscr{H}$ and $\eta\in\mathscr{N}_0$. Our goal is to show that $h$ is a linear combination of the functions $K(\cdot, x_i)$, that is $h\in\mathscr{H}_0$ where $\mathscr{H}_0$ is defined as the linear span of $K(\cdot, x_i)$ for $i=1,\cdots,N$. Further decompose $\mathscr{H} = \mathscr{H}_0 \oplus \mathscr{H}_0^{\perp}$ and $h = h_0 + h_1$ with $h_0\in\mathscr{H}_0$, $h_1\in\mathscr{H}_0^{\perp}$. By the reproducing property,
\begin{equation*}
	h_1(x_i) = \langle K(\cdot, x_i), h_1 \rangle = 0
\end{equation*}
Then we perform the following decomposition:
\begin{align*}
	\mathcal{L}(\hat{f}) &= \sum_{i=1}^N (y_i-h(x_i)-\eta(x_i))^2 + \lambda\cdot J(h+\eta) \\
	&= \sum_{i=1}^N (y_i-h(x_i)-\eta(x_i))^2 + \lambda\cdot J(h) \\
	&= \sum_{i=1}^N (y_i-h_0(x_i)-\eta(x_i))^2 + \lambda \cdot \langle h_0 + h_1, h_0 + h_1\rangle_{\mathscr{H}} \\
	&= \sum_{i=1}^N (y_i-h_0(x_i)-\eta(x_i))^2 + \lambda \cdot \langle h_0,h_0\rangle_{\mathscr{H}} + \lambda \cdot \langle h_1,h_1\rangle_{\mathscr{H}} \\
	&= \mathcal{L}(h_0+\eta) + \lambda\cdot J(h_1)
\end{align*}
which is greater than $\mathcal{L}(h_0+\eta)$ unless $J(h_1)=0$, that is $h_1\in \mathscr{H}\cap\mathscr{N} = 0$. $\qed$

\bigskip
After some preparation, Theorem \ref{representer} will directly apply to our problem.

\subsection{Proof of Theorem \ref{mainthm}}

\begin{proof}
The total space is $\mathscr{S} = W^{2,2}(\mathcal{M})$. This is a Hilbert space with the usual inner product of functions, however we instead use the \textit{Hessian product}
\begin{equation*}
	\llangle f,g \rrangle_H := \int_{\mathcal{M}} \sum_{i,j=1}^d \left( \frac{\partial^2 f}{\partial x_i \partial x_j}(p) \frac{\partial^2 g}{\partial x_i \partial x_j}(p)\right) d\mu(p)
\end{equation*}
which is well-defined for any set of local orthonormal coordinates $\{x_j\}$, so that $\llangle f,f \rrangle_H = \mathcal{H}^{\mathcal{M}}(f)$ as desired. The null space $\mathscr{N} = \mathcal{N}^{\mathcal{M}}$ consists of functions $\eta$ with everywhere vanishing Hessian, so that $\llangle f,\eta \rrangle_H=0$ and $\mathcal{H}^{\mathcal{M}}(f+\eta) = \mathcal{H}^{\mathcal{M}}(f)$ for all $f\in \mathscr{S}$.

Also, let $d\omega$ be the induced volume form and $N$ the outwards-oriented unit normal vector field on $\partial\mathcal{M}$. Define $\mathscr{H}$ as the subspace of all functions $f$ with $Nf=0$ on the boundary. Consider the partitioning of $\mathscr{H}$ into subsets $\mathscr{H}_c$ consisting of fixed boundary average,
\begin{equation*}
	\langle f \rangle_{\partial \mathcal{M}} := \frac{1}{\text{vol}( \partial\mathcal{M})} \int_{\partial\mathcal{M}} f d\omega =c, \quad c\in\mathbb{R}
\end{equation*}
which is well-defined due to the trace theorem. It is enough to prove the result for each $\mathscr{H}_c$ as the representer form of $f$ does not depend on the stratification. Furthermore, for each value of $\langle f \rangle_{\partial \mathcal{M}}=c$, we may transform the problem into that of $c=0$ by simply subtracting $c$ from each response value $y_i$, performing the minimization, and adding the constant function $c$ (which is in the span of $\zeta_1,\cdots, \zeta_{\nu}$) to the result. Thus, without loss of generality, we may additionally assume $\langle f \rangle_{\partial \mathcal{M}} = 0$ for the functions in $\mathscr{H}$. By passing to the universal cover, $Nf=0$ and $\mathcal{H}^{\mathcal{M}}(f) = 0$ together imply $f$ is constant, showing that $\llangle \cdot,\cdot \rrangle_H$ defines an inner product on $\mathscr{H}$ if $\partial\mathcal{M} \neq \varnothing$.

Recall Green's first identity:
\begin{equation*}
	\int_{\mathcal{M}} u\Delta_{\mathcal{M}} v d\mu + \int_{\mathcal{M}} \langle \nabla u, \nabla v\rangle d\mu = \int_{\partial\mathcal{M}} uNv d\omega.
\end{equation*}
With two applications of this formula, we can show that $\mathscr{H}$ is actually a RKHS with reproducing kernel a Green's function $G$ with respect to the biharmonic operator $\Delta_{\mathcal{M}}^2$,
\begin{equation}\label{rkhsbi}
	\Delta_{\mathcal{M}}^2 G(\cdot,q) = \delta_q
\end{equation}
By considering (\ref{rkhsbi}) as a system of Poisson equations $\Delta_{\mathcal{M}}G(\cdot,q) = F(\cdot,q)$ and $\Delta_{\mathcal{M}}F(\cdot,q) = \delta_q$, we may choose two boundary conditions for $G$ and $F$, respectively. For the former, we impose Dirichlet conditions $\gamma G(\cdot,q)=0$; for the latter, Neumann conditions. Note, however, by the divergence theorem
\begin{equation*}
	\int_{\mathcal{M}} \text{div} X d\mu = \int_{\partial\mathcal{M}} \langle X,N\rangle d\omega
\end{equation*}
applied to $X=\nabla F(\cdot,q)$, the constraint
\begin{equation*}
	\int_{\partial\mathcal{M}} NF(\cdot,q)(p) d\omega(p) = \int_{\mathcal{M}} \Delta_{\mathcal{M}}F(\cdot,q)(p) d\mu(p) = \int_{\mathcal{M}}\delta_q(p) d\mu(p)= 1
\end{equation*}
must hold. In particular, the simplest condition $NF(\cdot,q) = 0$ is prohibited. Instead, we require that
\begin{equation*}
	NF(\cdot,q) = \frac{1}{\text{vol}(\partial \mathcal{M})}, \quad\forall q\in \mathcal{M}
\end{equation*}

With the setup in place, for any $q\in\mathcal{M}$, we compute the following integral with local coordinates chosen around each $p\in \partial\mathcal{M}$ so that $\partial/\partial x_n$ is the unit normal.
\begin{align*}
	\llangle f, G(\cdot, q)\rrangle_H &= \int_{\mathcal{M}} \sum_{i,j=1}^d \left( \frac{\partial^2 f}{\partial x_i \partial x_j}(p) \frac{\partial^2 G(\cdot, q)}{\partial x_i \partial x_j}(p)\right) d\mu(p)\\
	&= \sum_{i=1}^d \int_{\mathcal{M}} \Big\langle \nabla \left( \frac{\partial f}{\partial x_i} \right)(p), \nabla \left( \frac{\partial G(\cdot, q)}{\partial x_i} \right)(p) \Big\rangle d\mu(p) \\
	&= \sum_{i=1}^d \int_{\partial\mathcal{M}} \left( \frac{\partial f}{\partial x_i} \right)(p) \; N \left( \frac{\partial G(\cdot, q)}{\partial x_i} \right)(p) d\omega(p) \\
	&\quad\quad- \int_{\mathcal{M}} \sum_{i=1}^d \left( \frac{\partial f}{\partial x_i} \right)(p) \cdot \Delta_{\mathcal{M}} \left( \frac{\partial G(\cdot, q)}{\partial x_i} \right)(p) d\mu(p) \\
	&= -\int_{\mathcal{M}} \big\langle \nabla f(p), \nabla \left( \Delta_{\mathcal{M}} G(\cdot, q) \right)(p) \big\rangle d\mu(p) \\
	&= \int_{\mathcal{M}} f(p) \cdot \Delta_{\mathcal{M}}^2 G(\cdot, q)(p) d\mu(p)- \int_{\partial\mathcal{M}} f(p) N(\Delta_{\mathcal{M}} G(\cdot, q))(p) d\omega(p) \\
	&= f(q).\color{white} \int
\end{align*}
The first boundary integral vanishes since $\partial f/\partial x_n=0$ and $\partial G(\cdot,q)/\partial x_i=0$ on the boundary for $i<n$. Also, since $N(\Delta_{\mathcal{M}} G(\cdot, q))$ is the uniform distribution on $\partial\mathcal{M}$, second boundary integral equals $\langle f \rangle_{\partial\mathcal{M}}=0$. Thus $\mathscr{H}$ is a RKHS with kernel $G$, and the Representer Theorem applies.

If $\partial\mathcal{M} = \varnothing$, instead simply define $\mathscr{H}$ as the orthogonal complement of $\mathcal{N}^{\mathcal{M}}$ in $W^{2,2}(\mathcal{M})$ with respect to the usual inner product, so that $\mathscr{T} = \mathscr{H} \oplus \mathcal{N}^{\mathcal{M}}$ is equal to $W^{2,2}(\mathcal{M})$. As $\mathscr{H} \cap \mathcal{N}^{\mathcal{M}} = \varnothing$, $\llangle \cdot,\cdot \rrangle_H$ defines an inner product on $\mathscr{H}$, and the above argument applies with the boundary integral terms vanishing trivially. Thus the Theorem is proved for both cases.
\end{proof}

\section{Hessian Splines}

\subsection{The Algorithm}
In this Section, we give an algorithm for computing splines given observations lying on some unknown submanifold of the feature space. The articulation manifold $\mathcal{M}$ and the locally isometric embedding $\phi: \mathcal{M} \rightarrow \mathbb{R}^n$ are now unknown. $N$ points $p_1,\cdots, p_N$ are i.i.d. sampled from a probability distribution $\mu$ on $\mathcal{M}$ which is strictly positive on the interior. (If there exist regions $\mathcal{R}$ of zero density, simply redefine the manifold $\mathcal{M}$ to be $\mathcal{M}\setminus \text{int}(\mathcal{R})$.) We observe the embedded data points $\mathbf{x}_i = \phi(p_i)$ and their corresponding response values $y_i = g(\mathbf{x}_i) + \epsilon_i$, $\epsilon_i \sim N(0,\sigma^2)$, where $g=f\circ \phi^{-1}$. Clearly, the manifold smoothing spline formulation of Section 3 is intractable, as we do not even know the space our functions are defined on.

One possibility is to first run a nonlinear dimensionality reduction scheme on the data, then apply the TPS algorithm to the resulting Euclidean coordinates. However, this requires combining two separate algorithms, increasing computational complexity and inheriting their various limitations. For example, Isomap is prone to introducing short-circuits in noisy data \cite{sax}, causing undesirable folds and creases which greatly impacts the performance of the proceeding TPS algorithm.

Instead, we focus on computable approximations to the penalty term $\mathcal{H}^{\mathcal{M}}$. In particular, the Hessian Eigenmaps algorithm gives an explicit construction of the matrix $\hat{\mathcal{H}}^{\mathcal{M}}(\mathbf{X})$ which approximates
\begin{equation}\label{esteq}
	\mathcal{H}^{\mathcal{M}}(f) \approx g(\mathbf{X})^{\top} \hat{\mathcal{H}}^{\mathcal{M}}(\mathbf{X}) g(\mathbf{X})
\end{equation}
for any pair of functions $f$ and $g=f\circ \phi^{-1}$. This allows for the estimation of $\mathcal{H}^{\mathcal{M}}(f)$ given only the observed values $f(p_i) = g(\mathbf{x}_i)$.

Recall that the algorithm requires $\mathcal{M}$ to be an open region of $\mathbb{R}^d$ and $\phi$ to be locally isometric. However, the construction itself does not depend on $\mathcal{M} \subset \mathbb{R}^d$, which only comes into play when computing the embedding coordinates $\{\theta^j\}$. Hence we only need the weaker condition that $\mathcal{M}$ is flat. This allows us to apply our algorithm to a larger class of topologically nontrivial spaces, such as the cylinder $\mathbb{R} \times S^1$, the flat torus $S^1 \times S^1$, and their higher-dimensional analogues. In contrast, most widely-used nonlinear dimensionality reduction algorithms require the assumption $\mathcal{M} \subset \mathbb{R}^d$ \cite{mlta}.

\bigskip
Thus, we are motivated to replace the loss $\mathcal{L}^{SS}(f;\lambda)$ by the ``Hessian spline'' loss
\begin{equation} \label{hsloss}
	\mathcal{L}^{HS}(g;\lambda) := \norm{\mathbf{y} - g(\mathbf{X})}_2^2 + \lambda g(\mathbf{X})^{\top} \hat{\mathcal{H}}^{\mathcal{M}}(\mathbf{X}) g(\mathbf{X})
\end{equation}
defined for functions $g: \phi(\mathcal{M}) \rightarrow \mathbb{R}$. The loss is quadratic and depends only on the values $g(\mathbf{X})$, so that the fitted values $\hat{g}^{HS}(\mathbf{X})$ may be computed simply as:
\begin{equation*}
	\hat{g}^{HS}(\mathbf{X}) = (I_N + \lambda \hat{\mathcal{H}}^{\mathcal{M}}(\mathbf{X}))^{-1} \mathbf{y},
\end{equation*}
a \textit{linear smoother} of $\mathbf{y}$, and out-of-sample estimates $\hat{g}(\mathbf{x})$ for $\mathbf{x}$ (assumed to be) lying on $\phi(\mathcal{M})$ may be calculated using neighborhood interpolation techniques. The smoothing parameter $\lambda$ may again be chosen using a form of cross-validation on a development set.

If we are additionally given weights $\mathbf{w}=(w_1,\cdots,w_N)^{\top}$ representing the reliability of each datum $y_i$, the weighted loss
\begin{equation*}
	\mathcal{L}^{HS}(g;\lambda,\mathbf{w}) = \sum_{i=1}^N w_i(y_i-g(\mathbf{x}_i))^2 + \lambda g(\mathbf{X})^{\top} \hat{\mathcal{H}}^{\mathcal{M}}(\mathbf{X}) g(\mathbf{X})
\end{equation*}
is minimized by
\begin{equation*}
	\hat{g}^{HS}(\mathbf{X}) = (\mathbf{W} + \lambda  \hat{\mathcal{H}}^{\mathcal{M}}(\mathbf{X}))^{-1} \mathbf{W} \mathbf{y},
\end{equation*}
where $\mathbf{W} = \diag (\mathbf{w})$. The weights may be predetermined according to the sampling process, or smaller values may be assigned to outliers in order to increase stability. In particular, the weights may be iteratively updated by predicting how likely each point is to be an outlier at each step. A data point $\mathbf{x}_i$ is more likely to be an outlier if its true response $y_i$ is far away from a preliminary fitted response. Pick a decreasing function $\rho: [0,\infty) \rightarrow [0,1]$, $\rho(0)=1$, representing how much a datum should be downweighted given this distance. A typical example is $\rho(r) = \exp(-r/2\hat{\sigma}_p)$ for a preliminary estimate $\hat{\sigma}_p$ of $\sigma$. The reweighting scheme is described in Algorithm 1.

\begin{algorithm}
	\caption{Iterative updating procedure for weights.}
	\KwData{predictors $\mathbf{X}$, responses $\mathbf{y}$}
	initialize $w_i=1$ for $i=1,\cdots,N$\;
	\While{$(w_i)$ has not converged}{
		compute current estimates $\hat{g}^{HS}(\mathbf{X})$\;
		\For{$i=1,\cdots,N$}{
			compute reliability $r_i=|y_i-\hat{g}^{HS}(\mathbf{x}_i)|$\;
			update $w_i \leftarrow \rho(r_i)w_i$\;
		}
	normalize $\mathbf{w}$ so that $\sum_{i=1}^N w_i = N$\;
	}
\end{algorithm}

It should be pointed out that while the Hessian spline estimator was motivated and studied on the premise of zero curvature, the algorithm itself (unlike Hessian Eigenmaps) does not require this condition per se. For any set of observations, it will nevertheless compute response values fitted with the loss (\ref{hsloss}). The penalty term now measures the Euclidean Hessian of the data projected to the tangent space, which should still give a reasonable measure of how curved the regression function is, unless the data has regions of disproportionately high curvature.

\subsection{Out-of-sample Prediction}

The Hessian spline does \textit{not} give an estimator for the regression function $g$ on the entire manifold, as the problem is no longer cast in a functional form. In particular, it is unable to say anything about the value of $g(\mathbf{x}^*)$ for $\mathbf{x}^*\neq \mathbf{x}_i$, $i=1,\cdots,N$. One method would be to consider the augmented data $\mathbf{X}^* = (\mathbf{X}, \mathbf{x}^*)$ and find the minimizer $y^*$ of the updated Hessian loss by solving the quadratic optimization problem
\begin{equation*}
	y^* = \argmin_{y\in\mathbb{R}}\, (g(\mathbf{X})^{\top}, y) \hat{\mathcal{H}}^{\mathcal{M}} (\mathbf{X}^*) (g(\mathbf{X})^{\top}, y)^{\top}.
\end{equation*}
However this would be needlessly computationally expensive, as the global Hessian estimator must be recomputed. Instead, we only take into account the neighborhood $N(\mathbf{x}^*)$ of the $K$ data points $\mathbf{x}_i$ closest to $\mathbf{x}^*$ (including $\mathbf{x}^*$ itself). As per the philosophy of manifold learning, the local patch $N(\mathbf{x}^*)$ approximates the tangent space $T_{\mathbf{x}^*}\mathcal{M} \simeq \mathbb{R}^d$, embedded as an affine subset of the ambient space via an isometry $\pi: \mathbb{R}^d\rightarrow  \mathbb{R}^n$. The $d$-dimensional local tangent coordinates of $N(\mathbf{x}^*)$ may be obtained via PCA or a robust version thereof \cite{rpca}.
Applying TPS to the transformed dataset (excluding $\mathbf{x}^*$) will then provide a functional estimator of $g$, defined on the entire region. Evaluating at the local coordinates corresponding to $\mathbf{x}^*$ will give an out-of-sample prediction minimizing local curvature.

An even more efficient method would be to suppose $g$ is locally linear and perform local linear regression on $N(\mathbf{x}^*)$. However, this method assumes that the distances between neighboring points are much smaller compared to the scale of curvature of $g$ (as a graph $\mathbb{R}^d \rightarrow \mathbb{R}$). If not, the predictions will always be positively biased if $g$ is convex at $\mathbf{x}^*$, and negatively biased if $g$ is concave.

\bigskip
Having a prediction scheme for unobserved data suggests an application of the spline algorithm as a supervised nonlinear classifier. Suppose we are given points $\mathbf{x}_i$ on an unknown manifold with known binary labels $L_i\in\{0,1\}$ as our training set. The algorithm estimates a smoothed regression function $\hat{g}^{HS}(\mathbf{X})$ that approximates $L_i$. Since it is a linear smoother, the estimated responses $\hat{L}_i$ will be contained in the interval $[0,1]$. We classify $\mathbf{x}_i$
$$\hat{c}(\mathbf{x}_i)= \begin{cases}
	0 \quad \hat{g}^{HS}(\mathbf{x}_i)\leq 0.5 \\
	1 \quad \hat{g}^{HS}(\mathbf{x}_i)> 0.5
\end{cases}$$
Similarly, given a test point $\mathbf{x}^*$, we compute the out-of-sample prediction $\hat{g}^{HS} (\mathbf{x}^*) = y^*$ and label $\mathbf{x}^*$ as
$$\hat{c}(\mathbf{x}^*)= \begin{cases}
	0 \quad \hat{g}^{HS}(\mathbf{x}^*)\leq 0.5 \\
	1 \quad \hat{g}^{HS}(\mathbf{x}^*)> 0.5
\end{cases}$$
 Using locally linear interpolation will guarantee $y^*\in [0,1]$, which may be useful when interpreting $y^*$ as the probability of labeling as 1, such as for simulation purposes. Also, the smoothing aspect of the algorithm ensures robustness against noisy data, i.e. many wrong labels. Increasing the hyperparameter $\lambda$ in this aspect will increase robustness while potentially losing information on the finer details of the dataset. In practice, $\lambda$ may be tuned by testing accuracy on a development set.

If a multiclass classifier is needed, with labels $C=\{c_1,\cdots,c_J\}$, we may construct for each $j=1,\cdots,J$ a binary classifier $\hat{g}_j^{HS}$ which classifies points as either $c_j$ or not-$c_j$ as above, and output the most likely label:
$$\hat{c}(\mathbf{x}^*) = \argmax_{j=1,\cdots ,J} \hat{g}_j^{HS}(\mathbf{x}^*).$$

Some other applications are boundary detection, smooth denoising, and random field estimation.

\subsection{Asymptotic Error}

We analyze the asymptotic properties of the Hessian loss approximation (\ref{esteq}), which provides a solid rationale for the use of the Hessian spline estimator.
The approximation consists of two parts: the second order Taylor expansion of $f$ around $\mathbf{x}_i$, and the replacement of the finite sum of the point Hessian norm values by the Hessian loss integral,
\begin{equation*}
	\frac{1}{N} \sum_i \lVert{\Hess_f^{\tan}(\mathbf{x}_i)\rVert}_F^2 \approx \int_{\mathcal{M}}\nolimits \lVert{\Hess_f^{\tan}(p)\rVert}_F^2 d\mu(p)
\end{equation*}
The latter is simply the law of large numbers:

\begin{thm}[Strong Law of Large Numbers]
	Let $\mathbf{x}_i$, $i=1,\cdots,N$ be $\mathcal{M}$-valued random variables sampled independently from a probability measure $d\mu$ on $\mathcal{M}$. Let also $\eta: \mathcal{M} \rightarrow \mathbb{R}$ be a measurable function, so that $\eta(\mathbf{x}_1)$ is Lebesgue integrable, that is the expected value
	\begin{equation*}
		\mathbb{E}\eta(\mathbf{x}_1) = \int_{\mathcal{M}} \eta(p)d\mu(p)
	\end{equation*}
	exists. Then the sample mean converges almost surely to the expectation:
	\begin{equation*}
		\overline{\eta(\mathbf{x})}_N := \frac{1}{N} \sum_{i=1}^N \eta(\mathbf{x}_i) \xrightarrow{\text{a.s.}} \mathbb{E}\eta(\mathbf{x}_1).
	\end{equation*}
\end{thm}

Convergence follows by taking the function $\eta(p)=\lVert{\Hess_f^{\tan}(p)\rVert}_F^2$, which is integrable since the Hessian of $f$ is bounded, either by smoothness and the compactness of $\mathcal{M}$, or any other appropriate $C^2$ regularity conditions.

The asymptotic behavior of this convergence is governed by the Central Limit Theorem:

\begin{thm}[Lindeberg-L\'evy Central Limit Theorem]
	In the above setting, if the variance also exists i.e. $\Var(\eta(\mathbf{x}_1))= \sigma_{\eta}^2 <\infty$, the limiting distribution of $\overline{\eta(\mathbf{x})}_N$ is given by
	\begin{equation*}
		\sqrt{N} \left(\overline{\eta(\mathbf{x})}_N- \mathbb{E}\eta(\mathbf{x}_1)\right) \xrightarrow{d} Z, \quad Z\sim N\left(0,\sigma_{\eta}^2\right)
	\end{equation*}
\end{thm}
where $\xrightarrow{d}$ denotes convergence in distribution. Thus, the approximation error is of order $O_p(N^{-1/2})$.

\bigskip
For the first part, the error due to the Hessian matrix approximation scheme (\ref{scheme}) in each entry consists of the third order Taylor remainder terms of $f$:
\begin{equation}\label{taylorerr}
	\frac{\partial^2 f}{\partial x^{\alpha} \partial x^{\beta}} (\mathbf{x}_i) - \sum_j  \mathbf{H}_{(\alpha \beta) ,\, j}^{(i)} f(\mathbf{x}_j) = \frac{1}{6} \sum_{k, \ell,m} \frac{\partial^3 f}{\partial x^k \partial x^{\ell} \partial x^m}(\mathbf{x}_{i,j}^*) \left( \sum_j \mathbf{H}_{(\alpha \beta) ,\, j}^{(i)} \epsilon_{j,k}^{(i)} \epsilon_{j,\ell}^{(i)} \epsilon_{j,m}^{(i)}\right)
\end{equation}
for some $\mathbf{x}_{i,j}^*\in U_{\mathbf{x}_i}$. The second- and third-order partial derivatives of $f$ are assumed to be uniformly bounded on $\mathcal{M}$. To evaluate the terms $\mathbf{H}_{(\alpha\beta)}^{(i)}$ and $\epsilon_j^{(i)}$, we require bounds on the size of the neighborhoods $U_{\mathbf{x}_i}$. As $n$ (data density) increases, the expected distances to the $K$ closest points decrease proportionally to $n^{-1/d}$,
\begin{equation*}
	\lVert\epsilon_j^{(i)}\rVert \leq \diam(U_{\mathbf{x}_i}) = O_p(N^{-1/d}) \quad \forall i=1,\cdots,N, \quad\forall\mathbf{x}_j\in N(\mathbf{x}_i)
\end{equation*}
keeping in mind that geodesic distance is approximated by Euclidean distance up to order 3. The same bound will thus hold for the vectors $\mathbf{u}_k^{(i)}$. Unfortunately, the normalization step
\begin{equation*}
	\mathbf{H}_{(\alpha \beta)}^{(i)\,\,\top} \left( \mathbf{u}_k^{(i)}*\mathbf{u}_{\ell}^{(i)} \right) = 2\delta_{k,\ell}^{\alpha, \beta}
\end{equation*}
then forces each term of $\mathbf{H}^{(i)}$ to be on the order of $N^{2/d}$. Nevertheless, the entire right-hand side of (\ref{taylorerr}) is still $O_p(N^{-1/d})$ (compare to the actual entries of $\hat{\mathcal{H}}^{\mathcal{M}}(\mathbf{X})$, which are of order $N^{4/d}$). Thus, the total error incurred in Hessian loss estimation is:
\begin{equation*}
	|\mathcal{L}^{HS}(f\circ \phi^{-1};\lambda) - \mathcal{L}^{SS}(f;\lambda)| = O_p(N^{-1/2}+N^{-1/d})
\end{equation*}
which, unless the data is one-dimensional, is simply $O_p(N^{-1/d})$. Thus, as with many manifold learning algorithms \cite{mlta}, we may circumvent the curse of dimensionality as the dimension of the feature space $n$ has been replaced by the intrinsic dimension $d$ in the exponent.
Also, since $\mathcal{L}^{SS}$ is uniquely minimized by the spline $\hat{f}^{SS}$ and $\mathcal{L}^{HS}$ is minimized by any function interpolating $\hat{g}^{HS}(\mathbf{X})$, it follows that
\begin{equation*}
	\sup_{i=1,\cdots,N} \left| \hat{g}^{HS}(\mathbf{x}_i)- \hat{f}^{SS}\circ \phi^{-1}(\mathbf{x}_i) \right| \leq \lVert\hat{g}_0^{HS} - \hat{f}^{SS}\rVert_{\infty} = O_p(N^{-1/d})
\end{equation*}
where $\hat{g}_0^{HS}$ is the element in the affine space of functions interpolating $\hat{g}^{HS}(\mathbf{X})$ which is closest to $\hat{f}^{SS}$. Thus the Hessian spline estimator is able to efficiently approximate the true spline function at the given data points.

\subsection{Robustness}

We discuss how sensitive our algorithm is to noise in both the response values $y_i$ and embedded predictors $\mathbf{x}_i$. Recall that our model is $\mathbf{y} = g(\mathbf{X}) + \epsilon$, with $\epsilon\sim N(0,\sigma^2 I_N)$ an $N$-dimensional normally distributed random variable. Thus, the weighted Hessian spline estimator is also normally distributed with expected value the ideal fit to the exact values $y_i=g(\mathbf{x}_i)$,
\begin{equation*}
	\mathbb{E}_{\mathbf{y}}\left( \hat{g}^{HS}(\mathbf{X})\right) = (\mathbf{W} + \lambda  \hat{\mathcal{H}}^{\mathcal{M}}(\mathbf{X}))^{-1} \mathbf{W} g(\mathbf{X}),
\end{equation*}
and variance
\begin{equation}\label{noisevar}
	\Var_{\mathbf{y}}\left( \hat{g}^{HS}(\mathbf{X})\right) = \sigma^2 \cdot (\mathbf{W} + \lambda  \hat{\mathcal{H}}^{\mathcal{M}}(\mathbf{X}))^{-1} \mathbf{W}^2 (\mathbf{W} + \lambda  \hat{\mathcal{H}}^{\mathcal{M}}(\mathbf{X}))^{-1}.
\end{equation}
Since $\mathbf{W}$ and $\hat{\mathcal{H}}^{\mathcal{M}}$ are both positive-definite, (\ref{noisevar}) immediately implies
\begin{equation*}
	\norm{\Var_{\mathbf{y}}\left( \hat{g}^{HS}(\mathbf{X})\right)}_2 \leq \sigma^2
\end{equation*}
with respect to the matrix 2-norm $\norm{\mathbf{A}}_2 = \sigma_{\max}(\mathbf{A})$. Noting that the 2-norm is equal to the spectral radius $\rho(\mathbf{A})$ for Hermitian $\mathbf{A}$ and making use of the following:

\begin{prop}
	For an $N\times N$ Hermitian matrix $\mathbf{A}$ and $N$-vector $v$, $|v^{\top}Av|\leq \rho(\mathbf{A})\cdot \norm{v}^2$. In particular, the diagonal elements satisfy $|A_{ii}| \leq \rho(\mathbf{A})$.
\end{prop}
\noindent
We conclude that the variance of each individual fitted response is also bounded by $\sigma^2$,
\begin{equation*}
	\Var_{\mathbf{y}}\left( \hat{g}_i^{HS}(\mathbf{X})\right) \leq \sigma^2 \quad \forall i=1,\cdots,N
\end{equation*}

Now suppose the spatial predictors are corrupted by noise. While a large portion of research into manifold learning focus on only perturbations constrained to lie on the embedded submanifold $\phi(\mathcal{M})$, this is an unrealistic assumption as the feature map $\phi$ can involve any sort of physical, computational or otherwise non-mathematical process which introduce errors in the ambient space. However, for Hessian Eigenmaps the off-manifold component turns out to be insignificant.

In particular, suppose that for fixed $j$, instead of $\mathbf{x}_j$ we observe $\mathbf{x}_j^*=\mathbf{x}_j+\delta\mathbf{x}_j$ for some small $\delta\mathbf{x}_j\in\mathbb{R}^n$, which introduces errors in the computation of $\mathbf{H}^{(i)}$ for each $\mathbf{x}_i$ whose neighborhood contains $\mathbf{x}_j$. The perturbation can be split into parallel $\delta\mathbf{x}_j^{\parallel}$ and orthogonal components $\delta\mathbf{x}_j^{\perp}$ with respect to the tangent space $T_{\mathbf{x}_i}\mathcal{M}$. The first step is to perform PCA on $N(\mathbf{x}_i)$ to obtain the approximate tangent space $\hat{T}(\mathbf{x}_i)$ and the corresponding tangent coordinates $\mathbf{u}_k^{(i)}$. Already the projection of $\mathbf{x}_j^*$ onto the local tangent space kills off most of $\delta\mathbf{x}_j^{\perp}$. That is, as the affine subspace $\hat{T}(\mathbf{x}_i)$ probabilistically converges to $T_{\mathbf{x}_i}\mathcal{M}$ (in the sense of projection operators),
\begin{equation*}
	\norm{\mathbf{u}_j^{(i)\,*} - \mathbf{u}_j^{(i)}} = o_p(\delta\mathbf{x}_j^{\perp})
\end{equation*}

However, the presence of $\delta\mathbf{x}_j^{\perp}$ itself skews the fitted subspace $\hat{T}(\mathbf{x}_i)$, and since the neighborhood size $K$ is fixed, this effect cannot be ignored even asymptotically. If $\delta\mathbf{x}_j^{\perp}$ can be assumed small and roughly i.i.d. for all $j$, $\hat{T}(\mathbf{x}_i)$ should suffice as an unbiased estimate. If instead there are a few extreme outliers, robust PCA methods should be preferred and may even retrieve exact results (compared to no $\delta\mathbf{x}_j$ perturbation):

\begin{thm}[Robust PCA with PCP \cite{rpca}]
	Given a $m\times n$ matrix $M$, let $\norm{M}_*=\sum_i \sigma_i(M)$ be the nuclear norm and $\norm{M}_1=\sum_{ij} |M_{ij}|$ the 1-norm of $M$. Suppose $M$ admits a decomposition into a low-rank matrix $L_0$ and a sparse matrix $S_0$. Then, under suitably weak assumptions, the `Principal Component Pursuit' (PCP) estimate solving the convex problem
	\begin{equation*}
		\text{minimize } \norm{L}_* + \lambda \norm{S}_1 \quad \text{subject to } L+S=M
	\end{equation*}
exactly recovers $S_0$ and $L_0$.
\end{thm}

The assumptions involve identifiability conditions and bounds on the rank, singular decomposition and sparsity of the involved matrices. No prior knowledge of $\text{rank}(L_0)$ or $\supp(S_0)$ is assumed, the entries of $S_0$ can have arbitrarily large magnitude, and no tuning of $\lambda$ is required. Such methods should alleviate PCA's susceptibility to outliers at a relatively small computational cost.

Hence, we now discount the off-manifold error term and assume $\delta\mathbf{x}_j=\delta\mathbf{x}_j^{\parallel}$ so that the incurred error in the projected tangent coordinates satisfies $\lVert \delta\mathbf{u}_j^{(i)}\rVert\lesssim \norm{\delta\mathbf{x}_j}$. The main problem lies in the subsequent step of Gram-Schmidt orthogonalization. The centered coordinates $\mathbf{u}_k^{(i)}$ are themselves bound by the neighborhood size $\diam(U_{\mathbf{x}_i})$, which decreases at a speed of $N^{1/d}$ for fixed $K$ as more data is accumulated. One solution is to use variable $K$ and instead fix neighborhood size $\epsilon$. Nevertheless, the local nature of the Hessian estimator requires a neighborhood sufficiently small as to be approximately linear. Applying Gram-Schmidt to vectors of norm $\epsilon$ and $\epsilon^2$ is already quite unstable. Although modified Gram-Schmidt can reduce the propagation of error, the stability of the resulting $\mathbf{H}^{(i)}$ matrix, whose entries are $O_p(N^{2/d})$, cannot be assured.
Even if the orthogonalization process does not have a large effect (as is the case for small $K$ or low-dimensional data), the results are unreliable without error bounds such as
\begin{equation*}
	\lVert \delta\mathbf{u}_j^{(i)}\rVert < C\lVert \mathbf{u}_j^{(i)}\rVert \quad \forall i:N(\mathbf{x_i}) \ni \mathbf{x}_j \quad \text{for some } 0<C<1
\end{equation*}
Since $\norm{\delta\mathbf{x}_j}$ ostensibly does not depend on the other data points or the chosen neighborhood size, such an assumption cannot be made. However, the following example shows that the Hessian estimator is untrustworthy for perturbations on the scale of $\epsilon$.

Suppose $\mathcal{M}=\mathbb{R}$ and we are given three points $x_0<x_1<x_2$ with the corresponding response values. In this case the Hessian estimator at $x_1$ is simply the finite difference coefficient of order 2,
$$\mathbf{H}^{(1)}f = \hat{f}''(x_1)= \frac{2}{x_2-x_0} \left( \frac{f(x_2)-f(x_1)}{x_2-x_1} - \frac{f(x_1)-f(x_0)}{x_1-x_0}\right)$$
which is clearly sensitive to changes in each $x_i$. If for example $f(x_2)>f(x_0)>f(x_1)$ and $x_2$ is perturbed so that $x_0<x_2^*<x_1$, then $\hat{f}''(x_1)$ changes from a large positive value to a large negative value.

The saving grace is that these estimates are averaged across $N$ data points on the entire manifold to produce the final estimate $\hat{\mathcal{H}}^{\mathcal{M}}(\mathbf{X})$, reducing the total error. Note one corrupted predictor only influences at most $K$ point estimates. A more robust averaging scheme may also be used, ruling out those $\mathbf{H}^{(i)}$ with large norm, however this may render areas of high curvature harder to detect.

\subsection{Dimensionality}

There is a peculiarity concerning dimension in the above analysis of robustness. Contrary to the wide range of detrimental phenomena commonly termed the curse of dimensionality, an increase in dimension may benefit stability of our algorithm, or any finite difference based method -- a `blessing of dimensionality,' see e.g. \cite{bless}. This essentially occurs because perturbations in a bounded set will take place further from the origin on average. Consider a point $U$ randomly sampled from the uniform distribution on the $d$-dimensional unit ball $B^d(1)$. The expected norm $E\norm{U}$ is equal to $\frac{d}{d+1}$. More generally, for $N$ i.i.d. samples $U_1,\cdots,U_N$, the expected minimum distance is given by
$$\mathbb{E}\left( \min_{i=1\cdots N} \norm{U_i} \right) =\frac{N}{d} \frac{\Gamma(N)\Gamma(1/p)}{\Gamma(N+1+1/p)}$$
which converges to 1 as $d$ increases, even if $N$ also increases at a rate of $O(d)$. A similar result holds for the distance between samples. This separation serves to stabilize finite differences. To illustrate, suppose we have two data points $x_1,x_2$ in $\mathbb{R}^d$, with $x_1$ at the origin and $x_2$ infinitesimally close, and $g(x_1)=g(x_2)=0$. Assume $x_2$ is perturbed so that it is now located at $U$ uniformly sampled from $B^d(\delta)$, and both $g(x_i)$ values are also corrupted by noise $\epsilon_i\sim N(0,\sigma^2)$. For simplicity, we consider the empirical finite difference coefficient of order 1,
$$T=\frac{\epsilon_1-\epsilon_2}{\norm{U}}$$
Since $\norm{U}$ is distributed on $[0,1]$ with density $f_U(u)=du^{d-1}$, the following statistics are easily obtained, which highlight the stabilization of error at higher dimensions.

(i) If $d=1$, $\mathbb{E}\,|T|=\infty$ and $\mathbb{E}\,T$ is undefined. If $d\geq 2$, $\mathbb{E}\,|T|=\displaystyle \frac{d}{d-1}\frac{2\sigma}{\delta\sqrt{\pi}}$ and $\mathbb{E}\,T=0$.
	
(ii) If $d\leq 2$, $\Var(T)=\infty$. If $d\geq 3$, $\displaystyle \Var(T)= \frac{d}{d-2}\frac{2\sigma^2}{\delta^2}$.

\end{document}